\documentclass{article}
\usepackage{etex}
\usepackage{spconf,amsmath,graphicx}
\usepackage{algorithm, algpseudocode}
\usepackage{endnotes}

\usepackage{epsfig,psfrag}
\usepackage{pst-all}
\usepackage{amsmath,amsthm,amssymb,amsfonts,upref,cite,epsf,color,bm}
\usepackage{graphicx}
\usepackage{color}
\usepackage{amsmath}
\usepackage{graphicx}
\usepackage{calc}

\usepackage{tikz}
\usepackage{pgfplots}

\newtheorem{theorem}{Theorem}[section]

\newtheorem{lemma}[theorem]{Lemma}

\newtheorem{assumption}{Assumption}
\usepackage{framed}

\newcommand\norm[2][\Tnorm]{\ensuremath{{\|#2\|}_{#1}}}

\newcommand{\set}[1]{\mathcal{#1}} 
\newcommand\defeq{:=}

\newcommand{\proccomp}[1]{\vx_{#1}[\cdot]}

\newcommand{\setT}{\mathcal{T}}
\newcommand{\subsets}[1]{\mathfrak{S}^{#1}_{\sparsity}}
\newcommand\vect[1]{\mathbf #1}

\newcommand{\vx}{\vect{x}}

\newcommand{\procdim}{d}  
\newcommand{\mA}{\vect{A}}  
 
\newcommand{\mC}{\vect{C}}

\newcommand{\mK}{\vect{K}}

\newcommand{\mP}{\vect{P}}

\usepackage[applemac]{inputenc}

\def\baselinestretch{1}
\renewcommand{\baselinestretch}{1.6}\small\normalsize
\def \expect {{\rm E} }
\def \prob {{\rm P} }
\def \twiddle[#1] {e^{-j \frac{2 \pi}{N}  #1 }}
\def \twiddleneg[#1] {e^{j \frac{2 \pi}{N}  #1 }}

\DeclareMathOperator*{\argmin}{arg\;min}

\DeclareMathOperator*{\trace}{Tr}

\DeclareMathOperator*{\diag}{diag}

\def\ML_est{\hat{\mathbf{x}}_{\text{ML}}}

\newcommand{\cig}{\mathcal{G}}
\newcommand{\timevar}{n}
\newcommand{\timeidx}{\timevar}

\newcommand{\samplesize}{N}

\newcommand\sparsity{s_{\rm max}}
\newcommand{\block}{B}
\newcommand{\blockIndex}{b}
\newcommand{\blockIndexSet}{\mathcal{B}_{\blockIndex}}
\newcommand{\blockLength}{L}

\newcommand{\constantkap}{\kappa}

\newcommand{\edges}{\set{E}}
\newcommand{\nodes}{\set{V}}

\newcommand{\sample}[1]{\vx[#1]}
\newcommand{\component}[1]{\vx_{#1}[\cdot]}

\newcommand{\CMX}[1]{\mC[#1]}
\newcommand{\tvx}{\tilde{\vx}}
\newcommand{\ACMTX}[2]{\mC_{\tilde{x}}[#1,#2]}

\newcommand{\blockInverse}[3]{\big ( #1[#2,#3] \big)^{-1}}
\newcommand{\ACMTXI}[2]{\blockInverse{\mC_{\tilde{x}}}{#1}{#2}}
\newcommand{\PMX}[1]{\mK[#1]}
\newcommand{\PMTX}[2]{\mK_{\tilde{x}}[#1,#2]}

\newcommand{\PMTXI}[2]{\blockInverse{\mK_{\tilde{x}}}{#1}{#2}}
\newcommand{\condv}[3] {V_{#1}^{(#2, #3)}}
\newcommand{\condco}[3] {\mathbf{V}_{#1}^{(#2, #3)}}
\newcommand{\bfep}{\bm \varepsilon}
\newcommand{\submatrix}[3] {#1_{\{#2, #3\}}}

\newcommand{\be}{\begin{equation}}
\newcommand{\ee}{\end{equation}}

\allowdisplaybreaks

\makeatother

\usepackage{setspace}

\newlength{\depthofsumsign}
\title{Learning Conditional Independence Structure for High-dimensional Uncorrelated Vector Processes}
\name{Nguyen Tran Quang and Alexander Jung 
}
\address{\normalsize Dept. of Computer Science, Aalto University, Finland; firstname.lastname@aalto.fi\\[-0.5mm]
}
\begin{document}
\maketitle
\pagestyle{plain}



\begin{abstract}
We formulate and analyze a graphical model selection method for inferring 
the conditional independence graph of a high-dimensional nonstationary 
Gaussian random process (time series) from a finite-length observation. 
The observed process samples are assumed uncorrelated over time 
and having a time-varying marginal distribution. 
The selection method is based on testing conditional variances obtained for 
small subsets of process components. This allows to cope with the 
high-dimensional regime, where the sample size can be (drastically) 
smaller than the process dimension. We characterize the required 
sample size such that the proposed selection method is successful 
with high probability. 

\end{abstract}

\vspace*{-1mm}
\section{Introduction}
\label{sec_intro} 

\vspace{-2mm}

Consider a zero-mean, $\procdim$-dimensional Gaussian discrete-time random process (time series)
\vspace*{-2mm}
\begin{equation} 
\vx[\timevar] \!\defeq\! \big(x_{1}[\timevar],\ldots,x_{\procdim}[\timevar]\big)^{T} \!\in\! \mathbb{R}^{\procdim}\mbox{, for } \timevar = 1,\ldots,\samplesize.
\vspace*{-2mm}
\end{equation} 
Based on the observation of a single process realization of length $\samplesize$, 
we are interested in learning the conditional independence graph (CIG)  
\cite{Dahlhaus2000,DahlhausEichler2003,BachJordan04,PHDEichler} 
of $\vx[\timevar]$. The learning method shall cope with the \emph{high-dimensional regime}, 
where the number $\procdim$ of process components is (much) l
arger than the number $\samplesize$ of observed vector samples 
\cite{ElKaroui08,Santhanam2012,RavWainLaff2010,Nowak2011,Bento2010,MeinBuhl2006,FriedHastieTibsh2008}. 
In this regime, accurate estimation of the CIG is only possible 
under structural assumptions on the process $\vx[\timeidx]$.
In this work, we will consider processes whose CIGs are 
\emph{sparse} in the sense of containing relatively few edges.
This problem is relevant, e.g., in the analysis of medical diagnostic data (EEG) \cite{Nowak2011}, 
climatology \cite{EbertUphoff2012} and genetics \cite{DavidsonLevin2005}. 

Most of the existing approaches to graphical model selection (GMS) for Gaussian 
vector processes are based on modelling the observed data either as i.i.d.\ 
samples of a single random vector, or as samples of a stationary random process.  
For nonstationary processes, the problem of inferring time-varying 
graphical models has been considered \cite{KolarXing,ZhouLafferty2008}. 
By contrast, we assume one single CIG representing the correlation 
structure for all samples $\vx[\timevar]$, 
which are assumed uncorrelated but having diifferent marginal distributions, 
which are determined by the covaraince matrix $\mathbf{C}[\timevar]$. 

\vspace{-5mm}

\paragraph*{Contributions:}
Our main conceptual contribution resides in the formulation of a simple GMS 
method for unorrelated nonstationary Gaussian processes, which is based 
on conditional variance tests. For processes having a sparse CIG, these tests 
involve only small subsets of process components. 
We provide a lower bound on the sample size which guarantees that the correct 
CIG is selected by our GMS method with high probability. This lower bound depends only 
logarithmically 
on the process dimension and polynomially on the maximum degree of the true CIG. 
Moreover, our analysis reveals that the crucial parameter determining the required 
sample size is the minimum partial correlation of the process. 



\vspace{-0.2cm}
\paragraph*{Outline:}
The remainder of this paper is organized as follows. 
In Section~\ref{SecProblemFormulation}, we formalize the considered process model 
and the notion of a CIG. In particular, we will state four assumptions on the class of 
processes that will be considered in the following. 
Section \ref{sec_GMS_via_cond_var_testing} presents 
a GMS method based on conditional variance testing. There, we 
also state and discuss a lower bound on the sample size guaranteeing 
success of our GMS method with high probability. 
\vspace{-0.5cm}

\paragraph*{Notation:} 
Given a $\procdim$-deminsional process $\vx[1],\ldots,\vx[\samplesize]$ or length $\samplesize$ ,  
we denote 
a scalar component process as $\mathbf{x}_{i}[\cdot] \defeq \big(x_{i}[1],\ldots,x_{i}[\samplesize]\big)^{T} \in \mathbb{R}^{\samplesize}$ for 
$i\in \{1,\ldots,d\}$. 
The Kronecker-delta is denoted $\delta_{n,n'}$ with $\delta_{n,n'}=1$ if $n=n'$ and $\delta_{n,n'}=0$ else. 
By $\subsets{r}$, we denote all subsets of $\{1,\ldots,d\}$ of size at most $\sparsity$ and which do not contain $r$. 
We denote by $\submatrix{\mA}{\mathcal{A}}{\mathcal{B}}$ the submatrix with 
rows indexed by $\mathcal{A}$ and columns indexed by $\mathcal{B}$. 
Given a matrix $\mA$, we define its infinity norm as $\norm[\infty]{\mA} \defeq \max_{i} \sum_{j} {|A_{i,j}|}$. 
The minimum and maximum eigenvalues of a positive semidefinite (psd) 
matrix $\mathbf{C}$ are denoted $\lambda_{\rm min}(\mC)$ and 
$\lambda_{\rm max}(\mC)$, respectively.

\section{Problem Formulation}
\label{SecProblemFormulation}
\vspace*{-2mm}

Let $\vx[\timevar]$, for $\timevar \in \{1, \ldots, \samplesize\}$, be a zero-mean $\procdim$-dimensional, 
real-valued Gaussian random process of length $\samplesize$. We model the time samples $\vx[\timevar]$ as 
uncorrelated, and therefore independent due to Gaussianity. The probability 
distribution of the Gaussian process $\vx[\timevar]$ is fully specified by 
the covariance matrices $\CMX{\timeidx}$ which might 
vary with $\timevar$. 
To summarize, in what follows we only consider processes conforming to the model
\vspace*{-3mm}
\begin{align}
\label{equ_proc_model} 
& \{ \vx[\timeidx] \}_{\timeidx=1}^{\samplesize}  \mbox{  jointly Gaussian zero-mean with } \nonumber \\ 
&  \expect \{ \vx[\timevar] \vx^{T} [\timevar'] \} = \delta_{\timevar,\timevar'} \CMX{\timeidx}. 
\vspace*{-3mm}
\end{align} 
The process model \eqref{equ_proc_model} is relevant for applications facing weakly dependent 
time series, so that samples which are sufficiently separated in time can be effectively considered as uncorrelated \cite{Hwang2014}. 
Moreover, the process model \eqref{equ_proc_model} can be used as an approximation for the discrete 
Fourier transform of stationary processes with limited correlation width or fast decay of the 
autocovariance function \cite{JuHeck2014,HannakJung2014conf,JungGaphLassoSPL,CSGraphSelJournal}. 
Another setting where the model \eqref{equ_proc_model} is useful are vector-valued locally stationary processes, 
where a suitable local cosine basis yields approximately uncorrelated vector processes \cite{Don96}.   

For our analysis we assume a known range within which the 
eigenvalues of the covariance matrices $\CMX{\timevar}$ are 
guaranteed to fall. 
\vspace*{-1mm}
\begin{assumption} 
\label{aspt_eig_val}
The eigenvalues of the psd covariance matrices $\CMX{\timeidx}$ are bounded as 
\vspace*{-3mm}
\begin{equation}
\label{equ_unif_bound_eig_val_CMX}
0 <  \alpha[\timevar] \leq \lambda_{\rm min}(\CMX{\timeidx}) \leq \lambda_{\rm max}(\CMX{\timeidx}) \leq  \beta[\timevar]
\vspace*{-3mm}
\end{equation} 
with known bounds $\beta[\timevar] \geq \alpha[\timevar] > 0$. 
\vspace*{-3mm}
\end{assumption}
It will be notationally convenient to associate with 
the observed samples $\vx[1],\ldots,\vx[\samplesize]$ the
the ``time-wise'' stacked vector 
\vspace*{-2mm}
\begin{equation}
\vx = (\sample{1}^T,\ldots, \sample{\samplesize}^T)^T \in \mathbb{R}^{\samplesize \procdim} \nonumber
\vspace*{-1mm}
\end{equation} 
and the ``component-wise'' stacked vector 
\vspace*{-1mm}
\begin{equation}
\tvx \defeq (\component{1}^T,\ldots, \component{\procdim}^T)^T  \in \mathbb{R}^{\samplesize \procdim}. \nonumber
\vspace*{-2mm}
\end{equation} 
We have, for some permutation matrix $\mP \in \{0,1\}^{\samplesize \procdim \times \samplesize \procdim}$, 
\vspace*{-1mm}
\begin{equation}
\label{equ_permutation_relation}
\tilde{\vx} = \mathbf{P} \vx.
\vspace*{-2mm}
\end{equation} 

For data samples $\vx[\timevar]$ conforming to \eqref{equ_proc_model}, 
the associated vectors $\vx$ and $\tilde{\vx}$ are zero-mean 
Gaussian random vectors, with covariance matrices 
\begin{equation} 
\mC_{x}\! = \!\expect\{ \vx \vx^{T} \}\\ \! = \! \begin{pmatrix} \CMX{1} & \cdots & \mathbf{0}  \\
   \vdots &\ddots &\vdots \\
    \mathbf{0}&  \cdots &  \CMX{\samplesize} \\
   \end{pmatrix},  \label{equ_big_cov_matrix_x}
\end{equation} 
and 
\begin{equation} 
\mC_{\tilde{x}}\!=\!\expect \{ \tvx \tvx^{T} \}\!=\! \begin{pmatrix} \ACMTX{1}{1} & \cdots & 
	\ACMTX{1}{\procdim} \\ 
	\vdots &\ddots &\vdots \\ 
	\ACMTX{\procdim}{1}&\cdots & \ACMTX{\procdim}{\procdim} \\
   \end{pmatrix}, \label{equ_m_C_tilde_x}
\end{equation} 
respectively. 
Due to \eqref{equ_permutation_relation}, we have 
\vspace*{-3mm}
\begin{equation} 
\mC_{\tilde{x}} = \mP \mC_{x} \mP^{T}. \label{equ_rel_permuted_vecs_cov}
\vspace*{-3mm}
\end{equation} 
Since the permutation matrix $\mP$ is orthogonal ($\mathbf{P}^{T} = \mathbf{P}^{-1}$), 
the precision matrix $\mK_{x} \defeq \mC_{x}^{-1}$ is also block diagonal with diagonal blocks 
$\PMX{\timevar} = (\CMX{\timevar})^{-1}$.
As can be verified easily, the $(a,b)$th $\samplesize \times \samplesize$ block $\PMTX{a}{b}$
of the matrix $\mK_{\tilde{x}}= \mP \mK_{x} \mP^{T}$ is diagonal: 
\begin{equation}
 \PMTX{a}{b} = \begin{pmatrix} \submatrix{\big(\mK[1]\big)}{a}{b}& \cdots & \mathbf{0}  \\
   \vdots & \ddots &\vdots \\
    \mathbf{0}   &\cdots &  \submatrix{\big(\mK[\samplesize]\big)}{a}{b} \\
   \end{pmatrix}.  \label{equ_submatrix_L_a_b}
\end{equation}


\subsection{Conditional Independence Graph} 

We now define the CIG of a $\procdim$-dimensional Gaussian process 
$\vx[\timeidx] \in \mathbb{R}^{\procdim}$ as an undirected simple graph
$\cig=(\nodes,\edges)$ with node set $\nodes = \{1, 2, \ldots, \procdim\}$. 
Node $j \in \nodes$ represents the process component $\component{j}=(x_{j}[1],\ldots,x_{j}[\samplesize])^{T}$. 
An edge is absent between nodes $a$ and $b$, i.e., $(a,b) \notin \edges$ 
if the corresponding process components $\vx_{a}[\cdot]$ and $\vx_{b}[\cdot]$ 
are conditionally independent, given the remaining components $\{ \component{r} \}_{r \in \mathcal{V} \setminus \{a,b\}}$. 
Since we model the process $\vx[\timevar]$ as Gaussian (cf.\ \eqref{equ_proc_model}), this conditional independence 
can be read off conveniently from the inverse covariance (precision) matrix $\mathbf{K}_{\tilde{x}} \defeq \mathbf{C}_{\tilde{x}}^{-1}$. 

In particular, $\vx_{a}[\cdot]$ are $\vx_{b}[\cdot]$ are conditionally independent, 
given $\{ \component{r} \}_{r \in \mathcal{V} \setminus \{a,b\}}$ if and only if $\PMTX{a}{b} =\mathbf{0}$ \cite[Prop.\ 1.6.6.]{Brockwell91}.
Thus, we have the following characterization of the CIG $\cig$ associated with the process $\vx[\timeidx]$: 
\begin{equation}
\label{equ_edge_absent_corr_operator}
 (a,b) \notin \edges  \mbox{ if and only if } \PMTX{a}{b} =\mathbf{0}.
\end{equation} 
Inserting \eqref{equ_submatrix_L_a_b} into \eqref{equ_edge_absent_corr_operator} yields, in turn,  
\begin{equation}
\label{equ_charac_CIG_indpendent_not_identical_case_12}
\hspace*{-2mm}(a,b) \!\notin\! \edges   \mbox{ if and only if }  \submatrix{\big(\PMX{\timevar}\big)}{a}{b} \!=\! 0 \mbox{ for all } \timeidx \!\in\! [\samplesize].  
\end{equation}
We highlight the coupling in the CIG characterization 
\eqref{equ_charac_CIG_indpendent_not_identical_case_12}: 
An edge is absent, i.e., $(a,b) \notin \edges$ only if the precision 
matrix entry $\submatrix{\big(\mK[\timeidx]\big)}{a}{b}$ 
is zero \emph{for all} $\timeidx \in \{1,\ldots,\samplesize\}$. 

We will also need a measure for the strength of 
a connection between process components 
$\vx_{a}[\cdot]$ and $\vx_{b}[\cdot]$ for $(a,b) \in \edges$. 
To this end, we define the \emph{partial correlation}
between $\vx_{a}[\cdot]$ and $\vx_{b}[\cdot]$ as 
\vspace*{-4mm}
\begin{align}
\rho_{a,b} & \defeq 
(1/\samplesize) \sum_{n=1}^{\samplesize} \alpha[\timeidx] \big[ \big( \mK[\timeidx] \big)_{a,b} / \big( \mK[\timeidx] \big)_{a,a} \big]^{2}. 
\label{equ_partial_correlation_def}
\vspace*{-4mm}
\end{align}
Inserting \eqref{equ_charac_CIG_indpendent_not_identical_case_12} into \eqref{equ_partial_correlation_def} shows that $(a,b) \!\notin\! \mathcal{E}$ implies $\rho_{a,b}\!=\!0$. 

Accurate estimation of the CIG for finite sample size $\samplesize$ (incuring unavoidable sampling noise) is only 
possible for sufficiently large partial correlations $\rho_{a,b}$ for $(a,b) \in \mathcal{E}$. 
\vspace*{-2mm}
\begin{assumption} 
\label{aspt_minimum_par_cor}
For any edge $(a,b) \in \mathcal{E}$, the partial correlation $\rho_{a,b}$ (cf.\ \eqref{equ_partial_correlation_def})
is lower bounded by a constant $\rho_{\rm min}$, i.e., 
\vspace*{-4mm}
\begin{equation}
(a,b) \in \edges \Rightarrow \rho_{a,b} \geq \rho_{\rm min}. 
\vspace*{-2mm}
\end{equation}
\end{assumption}

The CIG $\cig$ of a vector-process $\vx[\timevar]$ is fully characterized by the neighborhoods 
$\mathcal{N}(r) = \{ t \in \mathcal{V}: (r,t) \in \mathcal{E} \}$ of all nodes $r \in \mathcal{V}$. 
Many applications involve processes with these neighborhoods being small compared to the overall 
process dimension $\procdim$. The CIG is then called \emph{sparse} since 
it contains few edges compared to the complete graph. 
\vspace*{-3mm}
\begin{assumption} 
\label{aspt_cig_sparse}
The size of any neighborhood $\mathcal{N}(r)$, i.e., the degree of 
node $r$ is upper bounded as 
\vspace*{-3mm}
\begin{equation}
\label{equ_sparsity_max_degree}
|\mathcal{N}(r)| \leq \sparsity, 
\vspace*{-3mm}
\end{equation}
where typically $\sparsity \ll \procdim$. 
\vspace*{-3mm}
\end{assumption}

\subsection{Slowly Varying Covariance}
For several practically relevant settings, such as stationary processes with 
limited correlation width \cite{JuHeck2014,HannakJung2014conf,JungGaphLassoSPL,CSGraphSelJournal}
or underspread nonstationary processes \cite{GM_spectra02}, the observed processes can be well 
approximated by the model \eqref{equ_proc_model} with the additional property of a 
\emph{slowly varying} covariance matrix $\mC[\timeidx]$ \cite{KolarXing,ZhouLafferty2008}. 
\begin{assumption}
\label{aspt_slow_change}
For a (small) positive constant $\kappa$, 
\vspace*{-2mm}
\begin{equation}
\label{equ_covariance_difference}
\norm[\infty]{\CMX{\timevar_{1}} - \CMX{\timevar_{2}}} \leq \constantkap (|\timevar_{2} - \timevar_{1}| / \samplesize). 
\end{equation}
\end{assumption}
In view of \eqref{equ_covariance_difference}, for some  $n_{0} \!\in\! \{1,\ldots,\samplesize-\blockLength\}$ and 
blocklength $\blockLength$ such that $\kappa (\blockLength/\samplesize) \ll 1$, 
we may approximate $\blockLength$ consecutive samples 
$\sample{\timevar_{0}}, \sample{\timevar_{0}+1}, \ldots, \sample{\timevar_{0}+ \blockLength-1}$ as 
being i.i.d.\ zero-mean Gaussian vectors with covariance matrix 
$\mathbf{C}=(1/\blockLength) \sum_{\timevar=\timevar_{0}}^{\timevar_{0}+\blockLength-1} \CMX{\timevar}$. 
This suggests to partition the observed samples evenly into length-$\blockLength$ 
blocks $\blockIndexSet = \{(\blockIndex\!-\!1)\blockLength\!+\!1, \ldots, \blockIndex \blockLength \}$ 
for $b={1,\ldots,\block=\samplesize/\blockLength}$.\footnote{For ease of notation and 
without essential loss of generality, we assume 
the sample size $\samplesize$ to be a multiple of the blocklength $\blockLength$.}
We can approximate the covariance matrix of the samples within block $\mathcal{B}_{\blockIndex}$ 
using the sample covariance matrix 
\vspace*{-3mm}
\begin{equation}
\label{equ_sample_cov_matrix_block}
\widehat{\mC}[\blockIndex] = (1/\blockLength) \sum_{\timeidx  \in \blockIndexSet} \sample{\timeidx} \vx^{T}[\timevar].
\vspace*{-3mm}
\end{equation}

\section{GMS via Conditional Variance Testing}
\label{sec_GMS_via_cond_var_testing} 

We will now formulate and analyze a GMS method for a nonstationary 
process $\vx[\timevar]$ conforming to the model \eqref{equ_proc_model}. 
To this end, we will first show how the CIG of $\vx[\timevar]$ can be characterzed in terms 
of conditional variance tests. The GMS method implements 
these conditional variance tests using the covariance matrix estimate 
$\widehat{\mC}[\blockIndex]$ (cf.\ \eqref{equ_sample_cov_matrix_block}). 

\vspace*{-2mm}
\subsection{Conditional Variance Testing}
\vspace*{-2mm}

The characterization \eqref{equ_charac_CIG_indpendent_not_identical_case_12} for the 
CIG $\cig$ of the process $\vx[\timeidx]$ \eqref{equ_proc_model} seems convenient: We just 
have to determine the non-zero pattern of the precision matrices $\PMX{\timevar}$ 
and immediatly can estimate the edge set of the CIG $\cig$.  
However, the problem is in estimating the precision matrix $\PMTX{\timevar}$ in 
the high-dimensional regime where typically $\samplesize \ll \procdim$. 
In particular, in the high-dimensional regime, any 
reasonable a estimator $\widehat{\mC}_{\tilde{x}}$ for the covariance matrix $\CMX{\timevar}$ is singular, 
preventing the use of the inverse $\widehat{\mC}_{\tilde{x}}^{-1}[\timevar]$ as an estimate for $\PMX{\timevar}$. 

In order to cope with the high-dimensional regime, 
we will now present an approach to GMS via determining the neighborhoods 
$\mathcal{N}(r)$ for all nodes $r$ which exploits the sparsity of the CIG (cf.\ Assumption \ref{aspt_cig_sparse}). 
Our strategy for determining the neighborhoods $\mathcal{N}(r)$ 
will be based on evaluating the conditional variance
\begin{equation}
\label{equ_def_cond_variance}
\condv{x}{r}{\mathcal{T}} \defeq (1/\samplesize) {\rm Tr} \{ \condco{x}{r}{\mathcal{T}} \},
\end{equation}  
with the conditional covariance matrix 
\begin{equation}
\label{equ_cond_cov_matrix_V_r_T}
\condco{x}{r}{\mathcal{T}} \defeq { \rm cov} \big\{ \proccomp{r} \big| \{ \proccomp{t}\}_{t \in \mathcal{T}}  \big\}.
\end{equation}
Here, $\mathcal{T} \subseteq \mathcal{V} \setminus \{r\}$ is a subset of at most $\sparsity$ nodes, i.e., 
$|\mathcal{T}| \leq \sparsity$. We can express the conditional covariance matrix $\condco{x}{r}{\mathcal{T}}$ 
in terms of the covariance matrix $\mC_{\tvx}$ (cf.\ \eqref{equ_m_C_tilde_x}) as \cite[Thm.\ 23.7.4.]{Lapidoth09}
\begin{equation}
\label{equ_cond_var_expression_111}
\condco{x}{r}{\mathcal{T}} = \ACMTX{r}{r} - \ACMTX{r}{\mathcal{T}} \ACMTXI{\mathcal{T}}{\mathcal{T}} \ACMTX{\mathcal{T}}{r}. 
\end{equation} 
Note that the conditional covariance matrix $\condco{x}{r}{\mathcal{T}}$ 
depends only on a (small) submatrix of $\mC_{\tvx}$ constituted 
by the $\samplesize \times \samplesize$ blocks $\ACMTX{i}{j}$ for $i,j \in \mathcal{T} \cup \{r\}$.

Using the block diagonal structure of $\mC_{x}$ (cf.\ \eqref{equ_big_cov_matrix_x}), 
we can simplify \eqref{equ_cond_var_expression_111} to obtain the following 
representation for the conditional variance: 

\begin{lemma} 
The conditional variance $\condv{x}{r}{\mathcal{T}}$ \eqref{equ_def_cond_variance} satisfies 
\vspace*{-3mm}
\begin{equation}
\condv{x}{r}{\mathcal{T}} = (1/\samplesize) \sum_{\timeidx = 1}^{\samplesize} \frac{1}{{(\submatrix{(\mC[n])}{\setT'}{\setT'})^{-1}}_{\{r,r\}}} \label{equ_conditional_covariance_matrix}, 
\vspace*{-2mm}
\end{equation} 
with $\mathcal{T}' \defeq \{ r \} \cup \mathcal{T}$. 
\vspace*{-2mm}
\end{lemma}
\begin{proof}
Consider the subset $\setT = \{t_1, t_2, \ldots, t_k\}$, let $\vx_{\setT}[\timeidx] = (x_{t_1}[\timeidx], \ldots, x_{t_k}[\timeidx])^T$ 
and $\mP_{\setT}$ be the permutation matrix transforming 
$\vx_{\setT} \defeq \big ( (\vx_{\setT}[1])^T,\ldots,
\vx_{\setT}[\samplesize])^T \big )^T$
into 
$\tvx_{\setT} \defeq (\component{t_1}^T,\ldots, \component{t_k}^T)^T$,
i.e., $\tvx_{\setT} = \mP_{\setT}\vx_{\setT}$. 
The covariance matrix for $\tvx_{\setT}$ is obtained as $\ACMTX{\setT}{\setT}= \mP_{\setT} \mC_{\setT} \mP_{\setT}^T$, 
and, in turn since $\mP_{\setT}^{-1} = \mP_{\setT}^{T}$, $\ACMTXI{\mathcal{T}}{\mathcal{T}} = \mP_{\setT} (\mC_{\setT})^{-1} \mP_{\setT}^T$.

The conditional variance $\condv{x}{r}{\setT}$ is then given as
\begin{align}
\label{equ_condtional_variance_inid}
& \hspace*{-3mm}(1/\samplesize) \trace \big \{ \ACMTX{r}{r} \!-\! \ACMTX{r}{\setT} \ACMTXI{\setT}{\setT} \ACMTX{\setT}{r} \big \} \\ 
& \!=\! (1/\samplesize)  \trace \big \{ \ACMTX{r}{r} \!-\! \ACMTX{r}{\setT} \mP_{\setT}  (\mC_{\setT})^{-1} \mP_{\setT}^T \ACMTX{\setT}{r} \big \}. \nonumber
\vspace*{-3mm}
\end{align}
Due to the block-diagonal structure of $\mC_{\tilde{x}}$ (cf.\ \eqref{equ_m_C_tilde_x}), 
\begin{align}
\hspace*{-5mm}\ACMTX{r}{\setT} \mP_{\setT}   
\!=\!
\begin{pmatrix}
   \submatrix{\big(\mC[1]\big)}{r}{\setT} & \cdots & \mathbf{0} \\
   \vdots &\ddots &\vdots \\
   \mathbf{0} & \cdots & \submatrix{\big(\mC[\samplesize]\big)}{r}{\setT} \\
  \end{pmatrix}. \label{equ_block_diagonal_product_C_P_T}
\end{align}
Inserting \eqref{equ_block_diagonal_product_C_P_T} into \eqref{equ_condtional_variance_inid}, yields further 
\vspace*{-3mm}
\begin{equation}
\label{equ_condtional_variance_inid_final}
\begin{split}
\condv{x}{r}{\mathcal{T}}  & = (1/\samplesize)  \sum_{\timeidx = 1}^{\samplesize} \bigg ( \submatrix{(\mC[n])}{r}{r} -\\[-3mm]
&\submatrix{\big(\mC[\timeidx]\big)}{r}{\setT} \big (\submatrix{\big(\mC[\timeidx]\big)}{\setT}{\setT} \big )^{-1} \submatrix{\big(\mC[\timeidx]\big)}{\setT}{r} \bigg ).
\end{split}
\end{equation}
The expression \eqref{equ_conditional_covariance_matrix} for the conditional variance follows 
then from \eqref{equ_condtional_variance_inid_final} using the matrix inversion lemma for block matrices \cite[Ex. 2.2.4.]{BishopBook}. 
\vspace*{-4mm}
\end{proof}

Using the conditional variance $\condv{x}{r}{\mathcal{T}}$, we can 
characterize the neighborhoods $\mathcal{N}(r)$ in the CIG $\cig$ as: 
\begin{theorem}
\label{thm_cond_variance_properties}
For any set 
$\mathcal{T} \in \subsets{r}$:
\begin{itemize}
\item If $\mathcal{N}(r) \setminus \mathcal{T} \neq \emptyset$, we have 
\vspace*{-3mm}
\begin{equation}
\label{equ_relation_cond_v_N_r_case1}
\condv{x}{r}{\mathcal{T}}  \geq \rho_{\rm min} + (1/\samplesize) {\rm Tr} \big \{ \PMTXI{r}{r} \big \}.
\vspace*{-3mm}
\end{equation}
\item For $\mathcal{N}(r) \subseteq \mathcal{T}$, we obtain 
\vspace*{-3mm}
\begin{equation}
\label{equ_relation_cond_v_N_r}
\condv{x}{r}{\mathcal{T}}  = (1/\samplesize) {\rm Tr} \big \{ \PMTXI{r}{r} \big \}. 
\vspace*{-2mm}
\end{equation}
\end{itemize} 
\end{theorem} 
\begin{proof}
see Appendix.
\end{proof}
As an immediate consequece of Theorem \ref{thm_cond_variance_properties}, we can determine 
the neighborhood $\mathcal{N}(r)$ by a simple conditional variance test procedure: 
\begin{equation}
\label{equ_char_neigborhood_penalized_cond_var_one}
\mathcal{N}(r) = \argmin_{\mathcal{T} \in \subsets{r}} \condv{x}{r}{\mathcal{T}} + \rho_{\rm min} |\mathcal{T}|.
\end{equation} 

\subsection{The GMS method}
We now turn the procedure \eqref{equ_char_neigborhood_penalized_cond_var_one} 
into a practical GMS method 
by replacing $\condv{x}{r}{\mathcal{T}}$ in 
\eqref{equ_char_neigborhood_penalized_cond_var_one} with the estimate 
\vspace{-2mm}
\begin{equation} 
\label{equ_con_var_emp_test_neighbor}
\widehat{\condv{x}{r}{\mathcal{T}}} = (1/\block) \sum_{\blockIndex=1}^{\block} \frac{1}{{(\submatrix{(\widehat{\mC}[\blockIndex])}{\setT'}{\setT'})^{-1}}_{\{r,r\}}} 
\vspace{-3mm}
\end{equation}
using the sample covariance matrix $\widehat{\mC}[\blockIndex]$ (cf.\ \eqref{equ_sample_cov_matrix_block}) and $\mathcal{T}' \defeq \{ r \} \cup \mathcal{T}$.
\vspace*{-2mm}
\begin{algorithm}
\caption{GMS for uncorrelated nonstationary processes}
\label{alg:main_alg}
\textbf{Input:} $\sample{1},\ldots, \sample{\samplesize}$, $\rho_{\rm min}$, $\sparsity$, blocklength $L$
\begin{algorithmic}[1]
	\For{ each node $r$ in $\nodes$}
		\State  \hspace*{-3mm}$\widehat{\mathcal{N}}(r) \defeq \argmin_{\mathcal{T} \in \subsets{r}} \hspace*{-3mm}\widehat{\condv{x}{r}{\mathcal{T}}} + |\mathcal{T}| \rho_{\rm min}$ (cf.\ \eqref{equ_con_var_emp_test_neighbor})
	\vspace*{-1mm}
	\EndFor
	\State combine estimates $\widehat{\mathcal{N}}(r)$ by 
	``OR-'' or ``AND rule''
	\begin{itemize}
	\item  OR: $(i,j) \!\in\! \widehat{\mathcal{E}}$ if either $(i,j) \!\in\! \widehat{\mathcal{N}}(i)$ or $(i,j) \!\in\! \widehat{\mathcal{N}}(j)$ 
	\item AND: $(i,j) \!\in\! \widehat{\mathcal{E}}$ if $(i,j) \!\in\! \widehat{\mathcal{N}}(i)$ and $(i,j) \!\in\! \widehat{\mathcal{N}}(j)$ 
	\end{itemize} 
\end{algorithmic} 
\textbf{Output:} CIG estimate $\widehat{\mathcal{G}} = (\mathcal{V},\widehat{\mathcal{E}})$
\vspace*{0mm}
\end{algorithm}

For a sufficiently large sample size $\samplesize$, the CIG estimate $\widehat{\mathcal{G}}$ delivered by
Alg.\ \ref{alg:main_alg} coincides with the true CIG $\cig$ with high probability. 
\begin{theorem} 
\label{thm_sample_complexity}
There are constants $c_{1}, c_{2}$ depending only on $\{ \alpha[\timevar],\beta[\timevar]\}_{\timevar \in \{1,\ldots,\samplesize\}}$
such that for a sample size 
\vspace*{-3mm}
\begin{equation}
\label{equ_thm_samplesize_bound}
\samplesize \geq  c_{1}  \frac{\sparsity^{5/2}}{\rho_{\rm min}^{3}}(\log\frac{\kappa \sparsity^{7/2}}{\delta \rho^3_{\rm min}}  + \sparsity \log d)
\vspace*{-3mm}
\end{equation}
Alg.\ \ref{alg:main_alg} used with blocklength 
\vspace*{-3mm}
\begin{equation}
\blockLength = c_{2} \frac{\sparsity^2}{\rho_{\rm min}^{2}}(\log\frac{\kappa \sparsity^{7/2}}{\delta \rho^3_{\rm min}}  + \sparsity \log d), \nonumber
\vspace*{-3mm}
\end{equation}
delivers the correct CIG with prob.\ at least $1\!-\!\delta$, i.e., $\prob\{ \widehat{\mathcal{G}}=\cig \} \!\geq\! 1\!-\!\delta$
\end{theorem} 
A detailed proof of Theorem \ref{thm_sample_complexity} is omitted for space restrictions and will 
be provided in a follow up journal publication. However, the high-level idea 
is straightforward: If the maximum deviation 
\vspace*{-3mm}
\begin{equation} 
E=\max_{r \in \mathcal{V}, \mathcal{T} \in \subsets{r}}|\widehat{\condv{x}{r}{\mathcal{T}}}-\condv{x}{r}{\mathcal{T}}| \nonumber
\vspace*{-2mm}
\end{equation} 
is less than $\rho_{\rm min}/2$, Alg.\ \ref{alg:main_alg} 
is guaranteed to select the correct neighorhoods, i.e., $\widehat{\mathcal{N}}(r) = \mathcal{N}(r)$ for all nodes $r \in \mathcal{V}$, implying 
the selection of the correct CIG, i.e., $\cig = \widehat{\cig}$. 
For controlling the probability of the event $E \geq \rho_{\rm min}/2$, we apply 
a large deviation characterization for Gaussian quadratic forms \cite[Lemma F.1]{CSGraphSelJournal}.

The lower bound \eqref{equ_thm_samplesize_bound} on sample size $\samplesize$ 
stated by Theorem \ref{thm_sample_complexity}, 
depends only logarithmically on the process dimension $\procdim$ and polynomially on the maximum node degree $\sparsity$. 
Thus, for processes having a sufficiently sparse CIG (small $\sparsity$), 
the GMS method in Alg.\ \ref{alg:main_alg} delivers the correct CIG even in scenarios where 
the process dimension is exponentially larger than the available sample size. Moreover, the 
bounds \eqref{equ_thm_samplesize_bound} depends inversely on the minimum 
partial correlation $\rho_{\rm min}$, which is reasonable as a smaller partial 
correlation is more difficult to detect. Note that the quantity $\rho_{\rm min}$ 
occuring in \eqref{equ_thm_samplesize_bound} represents the average (over $\timevar$) of the marginal conditional correlations 
between the process components. 
\vspace*{-6mm}
\section*{Appendix: Proof of Theorem \ref{thm_cond_variance_properties}}
\vspace*{-3mm}
We detail the proof only for the neighborhood $\mathcal{N}(1)$ of the 
particular node $1$. The generalization to an arbitrary node is then straightforward. 

Let us introduce the weight matrices $\mathbf{L}_{1,r} \defeq \PMTXI{1}{1}  \PMTX{1}{r}$. 
According to \eqref{equ_edge_absent_corr_operator} we have $\mathbf{L}_{1,r} = 0$ 
for $ r \notin \mathcal{N}(1)$. 
Using elementary properties of multivariate normal distributions  (cf.\ \cite[Prop. 1.6.6.]{Brockwell91}), we have
%
the decomposition
\vspace*{-4mm}
\begin{equation}
\label{equ_innov_repr_comp_1}
\proccomp{1} = \sum_{r \in \mathcal{N}(1)} \mathbf{L}_{1, r} \proccomp{r}  + \bfep_{1} 
\vspace*{-5mm}
\end{equation} 
with the zero-mean ``error term'' $\bfep_{1} \sim \mathcal{N}(\mathbf{0},\condco{x}{1}{\mathcal{N}(1)})$ 
whose covariance matrix is $\condco{x}{1}{\mathcal{N}(1)}=  \PMTXI{1}{1}$. 
The identity \eqref{equ_relation_cond_v_N_r} is then obtained as 
\vspace*{-3mm} 
\begin{equation} 
\condco{x}{1}{\mathcal{T}} \stackrel{\eqref{equ_innov_repr_comp_1},\mathcal{N}(1)\subseteq \mathcal{T}}{=}  \condco{x}{1}{\mathcal{N}(1)} =  \PMTXI{1}{1}.
\vspace*{-2mm}
\end{equation} 
Moreover, by the projection property of conditional expectations (cf.\ \cite[Sec. 2.7]{Brockwell91}), 
the error term $\bfep_{1}$ in \eqref{equ_innov_repr_comp_1} 
is uncorrelated (and hence independent) with (of) the process 
components $\{ \proccomp{r} \}_{r \in \{2,\ldots,\procdim\}}$, i.e., 
\vspace*{-3mm}
\begin{equation} 
\label{equ_error_uncorr_T}
\expect \{ \proccomp{r} \bfep_{1}^{T} \} = \mathbf{0} \mbox{ for all } r \in \{2,\ldots,\procdim\}. 
\vspace*{-5mm}
\end{equation} 

Let us now focus on the conditional variance $\condv{x}{1}{\mathcal{T}}$ for a subset $\mathcal{T} \in \subsets{1}$ 
with $\mathcal{N}(1) \setminus  \mathcal{T} \neq \emptyset$, i.e., there is an index $j \in \mathcal{N}(1) \setminus \mathcal{T}$. 
We use the shorthands $\mathcal{P} \defeq \mathcal{T} \cup \mathcal{N}(1)$
and $\mathcal{Q} \defeq \mathcal{P} \setminus \{j\}$. Note that $\mathcal{T} \subseteq \mathcal{Q}$. 
For the conditional mean 
$\widehat{\proccomp{j}}\defeq \expect \big \{ \proccomp{1} \big | \{ \proccomp{r} \}_{r \in \mathcal{Q}} \big \}$,
we have the decomposition 
\vspace*{-3mm}
\begin{equation} 
\label{equ_repr_comp_j}
\proccomp{j} = \widehat{\proccomp{j}}  + \bfep_{j}. 
\vspace*{-2mm}
\end{equation} 
with the zero-mean ``error term'' 
$\bfep_{j} \sim \mathcal{N}(\mathbf{0},\mathbf{C}_{e,j})$ being uncorrelated with the components 
$\{ \proccomp{r} \}_{r \in \mathcal{Q}}$, i.e., 
\vspace*{-3mm}
\begin{equation} 
\label{equ_error_uncorr_Q}
\expect \{ \proccomp{r} \bfep_{j}^{T} \} = \mathbf{0} \mbox{ for all } r \in \mathcal{Q}. 
\vspace*{-3mm}
\end{equation} 
Moreover, the inverse covariance of $\bfep_{j}$ satisfies
\vspace*{-3mm}
\begin{equation} 
\label{equ_error_cov_j_e}
\mC^{-1}_{e,j} = \mK[j,j],
\vspace*{-3mm}
\end{equation} 
with $\mK = \blockInverse{\mC_{\tilde{x}}}{\setT'}{\setT'}$, where $\setT' =  \setT \cup \{j\}$. Since 
the blocks $\mC_{\tilde{x}}[a,b]$ of the matrix $\mC_{\tilde{x}}$ (cf.\ \eqref{equ_submatrix_L_a_b}), the 
matrix  $\mK[j,j]$ is diagonal with main-diagonal given by the values $\frac{1}{{(\submatrix{(\mC[n])}{\setT'}{\setT'})^{-1}}_{\{1,1\}}}$ 
which, together with Assumption \ref{equ_unif_bound_eig_val_CMX}, yields 
\begin{equation} 
\label{equ_bound_C_e_j_alpha}
\mC_{e,j} \succeq \diag\{ \alpha[\timevar] \}_{\timevar=1,\ldots,\samplesize}. 
\vspace*{-3mm}
\end{equation} 

Inserting \eqref{equ_repr_comp_j} into \eqref{equ_innov_repr_comp_1} yields 
\vspace*{-3mm}
\begin{align}
\label{equ_innov_repr_comp_2}
\proccomp{1} & =  \hspace*{-2mm}\sum_{r \in \mathcal{N}(1) \setminus \{j\}} \hspace*{-2mm}\mathbf{L}_{1, r} \proccomp{r} + \mathbf{L}_{1, j} \widehat{\proccomp{j}} + \mathbf{L}_{1, j}{\bm \varepsilon}_{j} + {\bm \varepsilon}_{1} \nonumber \\ 
& = \sum_{r \in \mathcal{Q}} \mathbf{M}_{r} \proccomp{r}+ \mathbf{L}_{1, j}{\bm \varepsilon}_{j} + {\bm \varepsilon}_{1} . \\[-8mm]
\nonumber
\end{align} 
Due to \eqref{equ_error_uncorr_T} and \eqref{equ_error_uncorr_Q}, the terms $\mathbf{L}_{1, j}\bfep_{j}$ and $ \bfep_{1}$ 
are both uncorrelated (and therefore independent due to Gaussianity) 
to all the components $\{ \proccomp{r} \}_{r \in \mathcal{Q}}$ and moreover are also mutually uncorrelated, i.e., 
$\expect \{ {\bm \varepsilon}_{r} \big({\bm \varepsilon}^{T}_{1},{\bm \varepsilon}^{T}_{j})\} = \mathbf{0}$ for all $r \in \mathcal{Q}$ 
and $\expect \{ {\bm \varepsilon}_{j} {\bm \varepsilon}^{T}_{1}\}  = \mathbf{0}$.
According to the law of total variance \cite{BillingsleyProbMeasure} and since $\mathcal{T} \subseteq \mathcal{Q}$, we have 
$\condv{x}{1}{\mathcal{T}} \geq \condv{x}{1}{\mathcal{Q}}$.
Therefore, we obtain the lower bound: 
\vspace*{-3mm}
\begin{align}
\label{equ_proof_cond_v_second_case}
\condv{x}{1}{\mathcal{T}} &\geq \condv{x}{1}{\mathcal{Q}}  \stackrel{\eqref{equ_def_cond_variance}}{=} (1/\samplesize) {\rm Tr} \{ {\rm cov}\{ \proccomp{1} | \{ \proccomp{r}\}_{r\in \mathcal{Q}} \}   \nonumber \\ 
& \hspace*{-10mm}\stackrel{\eqref{equ_innov_repr_comp_2}}{=}  (1/\samplesize) {\rm Tr} \{  \mathbf{L}_{1, j} \mC_{e,j}  \mathbf{L}^{T}_{1, j} +\condco{x}{1}{\mathcal{N}(1)}   \}  \\ 
& \hspace*{-10mm}\stackrel{\eqref{equ_bound_C_e_j_alpha},\eqref{equ_submatrix_L_a_b}}\geq  (1/\samplesize) \sum_{n=1}^{\samplesize} \alpha[\timeidx] \big[ \big( \mK[\timeidx] \big)_{a,b} / \big( \mK[\timeidx] \big)_{a,a} \big]^{2}+ \condv{x}{1}{\mathcal{N}(1)}   \nonumber 
\vspace*{-3mm}
\end{align} 
valid for any $\mathcal{T} \in \subsets{1}$ with $\mathcal{T} \neq \mathcal{N}(1)$.
We obtain \eqref{equ_relation_cond_v_N_r_case1} by combining \eqref{equ_proof_cond_v_second_case} with Asspt. \ref{aspt_minimum_par_cor}.

\renewcommand{\baselinestretch}{0.9}\normalsize\footnotesize

\bibliographystyle{IEEEtran}
\bibliography{/Users/ajung/work/LitAJ_ITC.bib,/Users/ajung/work/tf-zentral}

\end{document}